\documentclass{article}
\usepackage{amssymb}
\usepackage{amsmath}
\usepackage{amsthm}\usepackage{amsfonts, latexsym}
\usepackage{xcolor}
\usepackage{graphicx}

\usepackage{hyperref}
\hypersetup{
    colorlinks=true,
    linkcolor=blue,
    filecolor=magenta,      
    urlcolor=cyan,
}

\def\w{{\bf w}}
\def\x{{\bf x}}
\def\u{{\bf u}}
\def\v{{\bf v}}
\def\y{{\bf y}}
\def\z{{\bf z}}

\def\e{{\bf e}}
\def\1{{\bf 1}}
\def\0{{\bf 0}}
\def\ie{{i.e.}}
\def\Ie{{I.e.}}
\def\RR{{\mathbb R}}
\def\l{{\ell}}

\newtheorem{theorem}{Theorem}[section]
\newtheorem{proposition}[theorem]{Proposition}
\newtheorem{lemma}[theorem]{Lemma}

\newtheorem{remark}[theorem]{Remark}

\begin{document}
\title{ On a realization of motion and similarity group equivalence classes of labeled points in $\RR^k$ with applications to computer vision. 
}
\author{Steven B. Damelin \thanks{Department of Mathematics, University of Michigan, 530 Church Street, Ann Arbor, MI, USA.\, damelin@umich.edu.}, 
David L. Ragozin \thanks{Department of Mathematics, University of Washington, Seattle, WA 98195, USA,\, rag@uw.edu} and 
Michael Werman \thanks{ Department of Computer Science, The Hebrew University, 91904, Jerusalem, Israel.\, michael.werman@mail.huji.ac.il}}
\date{24  March, 2021.}
\maketitle

\begin{abstract}
We study a realization of motion and similarity group equivalence classes of $n\geq 1$ labeled points in $\RR^k,\, k\geq 1$  as a metric space with a computable metric. Our study is motivated by applications in computer vision.
\end{abstract}

Keywords: Orthogonal group, Similarity group,  Analysis on manifolds,  Data science,  Optimization, Control, Vision, Manifold learning, Motion group.

\section{Introduction.}
\setcounter{equation}{0}

\subsection{Visual-objects and Vision-groups.}

We will work in Euclidean space $\mathbb R^k$ for some fixed $k\geq 1$. 
\medskip

Visual-object recognition is the ability to perceive properties (such as shape, color and texture) of a "visual- object" in $\mathbb R^{k}$. Regardless of the object's position or illumination, the ability to effectively identify the object, makes it a visual-object. 

One significant aspect of visual-object recognition is the ability to recognize a visual-object across varying viewing conditions. These varying conditions include object orientation, lighting, object variability for example size, color and other within-category differences. Visual-object recognition includes viewpoint-invariant, viewpoint-dependent and multiple view theories.
 With this in mind, imagine we are given two visual-objects $O_{b}$ and $O_{b}'$ in $\mathbb R^k$. 
We think of $O_{b}$ and $O_{b}'$ as visually-equivalent if  there exists a well defined group action $g_{v}:O_{b}\to O_{b'},\, g_{v}\in G_{v}$ with $G_{v}$ a "vision-group".
\medskip

Some examples of actions $g_{v}$: \footnote{Let $f:\mathbb R^{k}\to \mathbb R^{k}$ be a map and suppose ${\rm det}(f')(x)$ exists on all of $\mathbb R^k$. 
Then $f$ is proper(orientation preserving) or respectively  improper if ${\rm det}(f')(x)>0,\, x\in \mathbb R^k$ or respectively ${\rm det}(f')(x)<0$, $x\in \mathbb R^k$.}

\begin{itemize}
\item[(a)] Affine maps: The map $A:\mathbb R^{k}\to \mathbb R^{k}$ is an affine map if there exists a linear transformation $M:\mathbb R^k\to \mathbb R^k$ and $\vec{t}\in \mathbb R^k$ so that for every $x\in \mathbb R^k$, $A(x)=Mx+\vec{t}$. Affine maps preserve area (volume) ratios. If $M$ is invertible (i.e., $A$ is then invertible affine), then $A$ is either proper or improper. If $M$ is not invertible, the map $A$ is neither proper or improper.
\item[(b)] Euclidean motions:  An affine map $A$ is an improper Euclidean motion  if $M\in O(k)$ and a proper  Euclidean motion if $M\in SO(k)$.  Euclidean motions can only be proper or improper. Here, $O(k)$ and $SO(k)$ are respectively the orthogonal and special orthogonal groups.
\item[(c)] Reflections:  A reflection $A:\mathbb R^{k}\to \mathbb R^{k}$  with respect to a hyperplane in $\mathbb R^k$ is an improper Euclidean motion with $A(x)=x$ on all of the hyperplane. 
\item[(d)] Similarity maps: This is a Euclidean motion plus a scaling. Similarity maps preserve length ratios. 
\end{itemize}

Here and throughout:
\begin{itemize}
\item[(1)] $GL(k)$ is the group of invertible linear maps in $\mathbb R^k=$ the group of non singular (real) $k\times k$  matrices.
\item[(2)] The orthogonal group $O(k)$. This  is the group of orthogonal real $k\times k$ matrices, ie $\left\{ A\in GL(k):\, A^TA=I_k\right\}$. Here and throughout, $I_k$ is the $k\times k$ identity matrix. 
\item[(3)] The special orthogonal group $SO(k):= \left\{O \in GL(k): O^{T}O=I_k \textnormal{ and } {\rm det}(O)=1 \right\}$.
\item With $\ltimes$ denoting semi direct product.
\item[(4)] The full translation group on $\RR^k$ is $I_k \ltimes \RR^k$.  
\item[(5)] The motion group on $\RR^k$ is $O(k)\ltimes \RR^k$.
\item[(6)] The affine group on $\RR^k$ is $Aff(k):=GL(k)\ltimes\RR^k$.
\end{itemize}

We refer the reader to the following references which give a good perspective of Section (1.1)  in various ways.
\cite{D6, D7, DR11, H65, Ki, Ki34, Ki51, Ki1, Lii1, O, O1, O2, W1, W8, W3}.

\subsection{$n$-pointed images.}

Suppose we have $n\geq 1$ labeled points (i.e. column vectors) ${\y}_1, {\y}_2,\ldots,{\y}_n\in \RR^k$ for some fixed $n$.
When these are assembled as the first, second \ldots n'th rows of an $n \times k$
matrix $Y:=[{\y}_1, ...., {\y}_n]^T\in \RR^{n\times k}$, we shall refer to $Y$ as an $n$-pointed image. Notice that an image $Y$ depends on the order and cardinality of the points. The rows 
of $Y$,  which we denote as
${\bf y}_i^T:=[y_{i,1}, ...y_{i,k}], 1\leq i\leq n$, represent the $k$ space coordinates of the $n$ image points ${\bf y}_1,...,{\bf y}_n$, though by taking them as rows we are really passing to the dual space, ${\RR^k}^*$, so actually $Y \in \RR^n \bigotimes {\RR^k}^* := M(n,k)$

A single image point $\y \in \RR^k$ is to be thought of as the {\it orthogonal} projection $P\Tilde{\y}$ of some point $\Tilde{\y}$ in a compact visual-object in $\RR^k$ with ${\rm Range}(P) \subseteq \RR^k$. Here, $e_i,\, 1\leq i\leq k+1$ are the basis vectors for $\mathbb R^{k+1}$. Our record of this image is formed by a "camera" on its "film" or "sensor" represented by the $\RR^k \bigoplus 0\e_{k+1}$ hyperplane.

In our $n$-pointed labeled image, $Y$, we are taking dual space images of i-labelled  points, ($1\leq i\leq n$)  under the adjoint/transpose $P^* = P^T = P$ where the last equality is due to the fact that $P$ is an orthogonal hence self-adjoint transformation/matrix. Moreover, since $P$ is a projection, $P^2 = P$. 

Thus, we have that the $n$ labeled points have been projected via $P$ to a $k$- dimensional affine subspace say $V$ with the projections $P$ of the visual-object's labeled points forming our $n$-pointed image. Note that $P(\RR^k)= V\textnormal{ and } P \circ P = P$. 

\subsection{Group invariance.}

The relevant geometric properties of our visual-objects and our images are assumed to be invariant under the action of a compact group, $G$ with the following minimal assumptions: 

\begin{equation}
SO(k)
\subseteq G \subseteq  GL(k)\ltimes \RR^k.
\label{e:Gprop}
\end{equation}

 Among others, this collection includes the group $O(k)$  and the proper similarity groups. 

\subsection{Goal.}

We refer the reader to Section (7.3), an appendix, which is needed moving forward.

Since ${\rm Range}(P)=V$, $P^2=P$ and $P=I_k$ on $V$, the projection $P$ commutes with $G$ on $V$.  Since $G$ effectively changes the coordinates 
in the camera/film or camera/sensor unit but has no effect on the coordinates of $\y$, the new film image coordinates will be given by $g^{-1}\y$, $g\in G$.
We extend this (rowwise) action of $G$ on each $\y_i^T$ to get an action of $G$ on each $Y \in M(n,k)$ given by
\[ \left( g \in\ G, Y \in M(n,k) \right) 
\mapsto
{Yg^{-1}*} .\]
When 
$ G \subseteq {GL(k)}$   
then this is just matrix multiplication.
When $G \subseteq 
Aff(k):=GL(k)\ltimes\RR^k$ then the action is given by 
\[
{g=\big((\tilde{g} \in GL(k), \u \in \RR^k), Y\in\RR^k \big)} \mapsto Y{\tilde{g}^{-1}}*-\1_n{{\tilde{g}^{-1}}*\u^T},
\]
\[
\textnormal{where }\1_n=[1,1,...,1]^T .
\] 

The goal of this paper is to understand and characterize the metric geometry of each orbit $GY$ and provide a computable metric realization of the space of all orbits.  That is 
$M(n,k)/G$.  Said another way, we want to analyze the space of $G$-orbits 
in $M(n,k)={\RR}^n \otimes {{\RR}^k}^*$ where for an affine subspace $V$, it's dual space $V^*$ can be identified as $V^T$ via the bilinear pairing $(\v,\u^T)\mapsto Tr(\v\u^T) \in \RR$.

Notice that above we have idealized the notion of film or sensor by giving it an infinite extent as a affine hyperplane $\RR^k \subset \RR^{k+1}$. Similarly there may be different idealized notions of exactly how our camera and lens forms an image as well as how it can move around in the space, $\RR^{k}$, containing our object.  
Here, though as we have already stated, we assume that these idealized notions are captured by a fixed group $G$.

We are ready to state our three main results.  As we discover, the geometry of each orbit $GY$ is based on an ellipsoid in $\mathbb R^k$ with  certain parameters which completely determine the geometry of the orbit for $GY$ for each $Y$. Here $<.>$ denotes inner product.

\section{Three main results.}
\setcounter{equation}{0}

Our first main result deals with the motion group case.

\begin{theorem}
Suppose we have $n$ labelled image points ${\y}_1, {\y}_2,...,{\y}_n\in \RR^k$. Suppose $n\geq k$.  Let the rows of the
matrix $Y:=[{\y}_1, ...., {\y}_n]^T\in \RR^{n\times k}$ which we will denote as ${\y}_i:=[y_{i,1}, ..., y_{i,k}]\in \RR^k, 1\leq i\leq n$, represent the $k$ space coordinates of the $n$ image points ${\y}_1,...,{\y}_n$.  Our goal is to realize the space of all equivalence classes of these $n$ pointed images as a metric space with a computable metric. We do this as follows:
Recall that $\1_n=[1,1,...,1]^T$. Now define $Y_{\rm {norm}}=(I-1/n 1_n1_n^T)Y$, $A_Y=Y_{\rm {norm}}Y_{\rm {norm}}^T$ (which is positive semi definite). Consider the map
\[
Y\mapsto E_Y=\left\{{\bf x}\in \mathbb R^k:\, <{\bf x},A_Y{\bf x}>=1\right\}.
\]  
\begin{itemize}
\item[(a)] This map from images to ellipsoids of dimension $k$ contained in $I_n^{\perp}$ maps ONTO the collection of all such ellipsoids.
\item[(b)] This map maps motion group equivalent images to the same ellipsoid.
\item[(c)] If two images map to the same ellipsoid, they are motion equivalent.
\end{itemize}
\label{t:main1}
\end{theorem}

The assumption $n\geq k$  is not ideal for several data applications, see for example \cite{D6}. To this end, we have our next main result dealing with the motion group case.

\begin{theorem}
Let $s:={\rm min}(n,k)$. Suppose we have $n$  labeled image points ${\y}_1, {\y}_2,...,{y}_n\in \RR^{s}$.  Let the rows of the
matrix $Y:=[{\y}_1,\ldots,{\y}_n]^T\in \RR^{n\times s}$ which we will denote as ${\y}_i:=[y_{i,1}, ...y_{i,s}]^T\in \RR^s,\, 1\leq i\leq n$, represent the $s$ space coordinates of the $n$ image points ${\bf y}_1,...,{\bf y}_n$.  Our goal is to realize the space of all equivalence classes of these $n$ pointed images as a metric space with a computable metric. We do this as follows:
Consider the map
\[
Y\mapsto E_Y=\left\{{\bf x}\in \mathbb R^s:\, <{\bf x},A_Y{\bf x}>=1\right\}.
\]  
\begin{itemize}
\item[(a)] This map from images to ellipsoids of dimension $s$ contained in $I_n^{\perp}$ maps ONTO the collection of all such ellipsoids.
\item[(b)] This map maps motion group equivalent images to the same ellipsoid.
\item[(c)] If two images map to the same ellipsoid, they are motion equivalent.
\end{itemize}
\label{t:main2}
\end{theorem}

Our third main result is the similarity group case.

\begin{theorem}
Assume the hypotheses of Theorem~\ref{t:main1} or Theorem~\ref{t:main2} and in addition, restrict attention to images with $ Y_{\rm {norm}}\neq 0$.
\begin{itemize}
\item[(a)] Then for real $a$, for which $E_{aY}$ is well defined, we have $E_{aY}=a^2E_{Y}$ and so the class of images equivalent to $Y$  maps ONTO the "line" of non- trivial ellipsoids which are positive multiples of $E_Y$.
\item[(b)] Choose now a "normalized" representative for those lines of ellipsoids such as:
\begin{itemize}
\item[(i)] The longest principal axis length=1.
\item[(ii)] The mean principle axis length=1.
\item[(iii)] The geometric mean principle axis length=1.
\end{itemize}  

This map then maps similarity group equivalent images to the same normalized ellipsoid.
\item[(c)] If two images map to the same normalized ellipsoid, they are similarity group equivalent.
\end{itemize}
\label{t:main3}
\end{theorem}

The remainder of this paper establishes Theorem~\ref{t:main1}, Theorem~\ref{t:main2} and Theorem~\ref{t:main3}.

\section{Orbits.}

In our quest for a metric characterization of $M(n,k)/G$ we must understand the action of $G$ and the geometry of each orbit $GY$. Note, $(A,\u)\in G \subseteq {GL(k)\ltimes \RR^k}$ acts on $M(n,k)$ from the right by 
$Y \in M(n,k) \mapsto Y(A^{-1},-A^{-1}\u) = YA^{-1} -\1_n \left(A^{-1}\u\right)^T$

The geometry of a single orbit $GY$ for any of these groups is then determined by the following fact.
\begin{lemma}
Fix an $n$-pointed image $Y$ in $\RR^k$. Let $G_Y = \left\{(A,\u):Y(A,\u)^{-1} = Y\right\}$ be the subgroup of $G$ fixing $Y$. Then the geometry of the orbit is exactly the geometry of the quotient space $G/G_Y$.    
\end{lemma}
We note that the orbit $GY$  intersects the orbit $GZ$ only if $g_1(Y) = g_2(Z)$ for some $g_1,g_2 \in G$. Thus $Z=g_2^{-1}g_1(Y)$ and so $GY=GZ$. We shall refer to each orbit as a $G$-equivalence class. Then our goal is to realize the quotient space $M(n,k)/G=\{GY:Y\in M(n,k)\}$ of all $G$-equivalence classes of these 
$n$-pointed images.

We now look at $G=SO(k)$ or $G=O(k)$ and study the  geometry of some  dense open subset of the orbit space for the $\tilde{G}=SO(k)$ case. We  provide a complete description of $G_Y$ and hence a geometric description of each orbit $G/G_Y$ based on the (non-negative) eigenvalues of the postive semi-definite $Y^TY$. The eigenvalues are ordered by size and their multiplities.

We have:

\begin{proposition}
Let $G=SO(k)\ltimes \RR^k$. A dense open subset of $M(n,k)/G$ is formed by the set $\{Y\in M(n,k): rank(Y)=min(n,k):=l\}/G$ and the $l$ non-zero eigenvalues of $Y^TY$
are distinct. Moreover each
\[
G_Y = \left\{ g =
 \begin{bmatrix} 
	       A & \v \\
	{\1_k}^T & 1 
  \end{bmatrix} : 
  \begin{bmatrix}
	   Y \\
	{\1_k}^T
\end{bmatrix} = \begin{bmatrix}
Y \\ {\1_k}^T
\end{bmatrix}g^{-1}\right\}.
\]       
\end{proposition}

\begin{proof}
We can choose $V \in SO(k)$ which diagonalizes $Y^TY$, \ie \begin{eqnarray*}
VY^TYV^{-1}
&&={\rm diag}(a_1^2,...,a_l^2,0,\ldots,0) \\  
\textnormal{ and } && |a_1| > \ldots> |a_l| > 0.
\end{eqnarray*}
Thus the first $l$ rows of $V^{-1} = V^T =[\v_1,\v_2,\ldots,\v_l,\ldots]^T$ are eigenvectors for $Y^TY$ for its distinct non-zero eigenvalues. 
\end{proof}

\begin{remark}
This remark deals with the  possibility of  ${\rm rank }(YY^T)<{\rm min}\left\{n,k\right\}$.  In this case:
\begin{itemize}
\item[(1)] The set $J:=\{Y: {\rm rank }(YY^T)<{\rm min}\left\{n,k\right\}\}$ is open and dense in $M(n,k)$, while its complement, $J^{C}$  is closed and lower dimensional.
\item[(2)] The  sets  $J$ and $J^{C}$ are $G$ invariant, so their images in the orbit space $M(n,k)/G$ are, respectively, open dense and lower dimensional. Hence the collection of 
$\{Y: {\rm rank }(YY^T)<{\rm min}\left\{n,k\right\}\}$  has measure zero in the space of orbits $M(n,k)/G$.
\end{itemize}
\end{remark}

\subsection{Eliminate the dependence on the pure translation.}

We may eliminate the dependence on the pure translation portions of the groups $G$.  We shall regularly use the definition $$\y_{cg}:= {1 \over n}\Sigma_i^n \y_i$$
for the center of gravity of $n$-labelled points.

Our required elimination is given in the following proposition.

\begin{proposition}
For $G$ any affine group with $I_k\ltimes\RR^k \subseteq{G} \subseteq{GL(k)\ltimes\RR^k}$, let $\tilde{G}:=\{A \in GL(k): (A,\u) \in G \textnormal{ for some } \u \in \RR^k \}$. Further let $ P_{\1_n}:=\1_n {\frac{1}{n}} \1_n^T$ be the orthogonal projection of $M(n,k)$ onto $\1_n {\RR^k}^T$ and let $M(n,k)_{norm} := (I_n - P_{\1_n})M(n,k)$ be the subspace of all $Y \textnormal{ such that } {\frac{1}{n}} \1_n^T Y = \0_k^T$, i.e. the row average of $Y$ is $\0_k$.

\begin{itemize}
\item[(1)] Then $\tilde{G}$ is a subgroup of $GL(k)$ and $G=\tilde{G}\ltimes \mathbb R^k$. Also
$SO(k)\subseteq \tilde{G} $.    
\item[(2)] $Y_{norm}$ is the unique point on the orbit $(I_n \ltimes \RR^k)Y =\{ Y -\1_n \u^T : \textnormal{For any } \u \in \RR^k \}$ with $P_{1_n}(Y-\1_n \u^T)=\0_{n \times k}$.
\end{itemize}
\end{proposition}

\begin{proof}
We prove first that $SO(k)\subseteq  \tilde{G}$. To see this, choose $A\in SO(k)$. Then
since $SO(k)\ltimes \RR^k\subseteq G$ by assumption, we have that $(A,{\bf u}) \in G$ for all ${\bf u}\in \mathbb R^k$. Thus 
$A\in \tilde{G}$. Next take $A,B\in \tilde{G}$. Then $(A,{\u})$ and $(B
,{\w})$ are in $G$ for some
${\u}$ and ${\w}\in \RR^k$. Thus since $G$ is a subgroup of the affine group of $\RR^k$, $(AoB,A{\w}+{\u})$ and $(A^{-1}, -A^{-1}\u)$ are both in $G$ but this then implies that $AoB$ and $A^{-1}$ are in $\tilde{G}$ and so $\tilde{G}$ is a subgroup of  the affine group of $\mathbb R^k$. Next, we must show that  $G=\tilde{G}\ltimes \mathbb R^k$. Firstly it is easy to see that $G\subseteq  \tilde{G}\ltimes\RR^k$. To see the reverse inequality, notice that  if ${\bf u}\in \mathbb R^k$ and $A\in \tilde{G}$ then
$(I,-{\u})o(A,{\u})\in G$ which means that for any $A\in \tilde{G}$, $(A,0)\in G$. Thus we have (1) $\tilde{G}\times \left\{0\right\}\subseteq G$ and (2) $I_k\times \RR^k\subseteq O(k)\bowtie \mathbb R^k\subseteq G$.
Thus $(\tilde{G}\times \left\{0\right\})o(I\times \RR^k)\subset G$ and this easily implies that  $\tilde{G}\ltimes \RR^k
\subseteq G$. Thus (1) is settled. For (2), note that $$Y-\1_n \u^T =Y_{norm} + P_{\1_n}Y -\1_n \u^T ,$$ so $$P_{1_n}(Y-\1_n \u^T) = \0_{n\times k} + \1_n\left({1 \over n}\Sigma_i^n \y_i^T - \u^T \right) .$$
Hence only when $\u^T = {1 \over n}\Sigma_i^n \y_i^T $ is $P_{1_n}(Y-\1_n\u^t) =\0_{n\times k}$. $\Box$.
\end{proof}

\section{Ellipsoids}
\setcounter{equation}{0}

We recall that we have already noted that our study of the  geometry of each orbit $GY$  when 
$G=SO(k)$ or $O(k)$ will be based on an ellipsoid in $k$-space determined by the positive semi-definite matrix $A_Y=Y^TY$.  Indeed, more precisely, we are going to use a singular value decomposition (SVD) to show that the parameters characterizing the ellipsoid \[
E_Y := \{\x\in\RR^k:<\x,A_Y\x>\leq 1\}\] completely determine the geometry of the 
orbit for $GY$ for each $Y$.

We will use: 

\begin{proposition}
Let $G$ be any closed subgroup of $O(k)$ with $SO(k) \subseteq G$ and let $Y \in M(n,k)$ have $Rank(Y)=\l$. Then the ellipsoid $E_Y \subseteq \RR^k$ determined by the unit sphere with respect to the Euclidean semi-norm $\left|\left|{\x}\right|\right|_{YY^T}^2 = <\x,A_Y\x>$ satisfies the following:
\begin{enumerate}
\item[i)] The $k$-semi-axes of $E_Y$ have lengths $$a_1^2\geq a_2^2\geq ...\geq a_{\l}^2\geq \ \ldots a_k^2\geq 0$$ and if $\l<k$ then $a^2_{\l}>0$ and 
$0 = a^2_{\ell+1}=\ldots=a^2_{k}.$
\item[ii)] The multiplicities of the semi-axes lengths $\sigma_1,...,\sigma_j$ for $j={\rm Card}\{a_{i}^2:1\leq i \leq k\}$,  satisfy 
$$\sigma_i={\rm diam}({\rm ker}((A_Y-a^2_{i})I_i)),1\leq i\leq j.$$
\item[iii)] Let $\RR^k =\bigoplus_{i=1}^j \RR^{\sigma_i}$. There exists an orthonormal basis for $\RR^k$, 
$B:=\{\v_{1},,,\v_{\sigma_1},\v_{2},,,\v_{\sigma_2},,, \v_{j},,,\v_{\sigma_j}\}$ such that each $$\textnormal{B}_i=\{\v_{i},\ldots,\v_{\sigma_i}\},\, 1\leq i\leq j$$
is an orthonormal basis for the $\RR^{\sigma_i}$ term in the orthogonal direct sum and 
the orthogonal matrix  $$g_{diag} =\begin{bmatrix} \v_{1},,,\v_{\sigma_1},\v_{2},,,\v_{\sigma_2},,, \v_{j},,,\v_{\sigma_j} \end{bmatrix} \in G$$ diagonalizes $A_Y$, i.e. $$g_{diag}^{-1}A_Yg_{diag} = \begin{bmatrix} a^2_1{I_{\sigma_1}} & 0\1_{\sigma_1}\1_{k-\sigma_1-\sigma_j}^T & 0\1_{\sigma_1}\1_{\sigma_j}^T \\
0\1_{k-\sigma_1-\sigma_j} \1_{\sigma_1}^T&\ddots & 0\1_{k-\sigma_1-\sigma_j}\1_{\sigma_j}^T \\
0\1_{\sigma_j}\1_{\sigma_1}^T & 0\1_{\sigma_j}\1_{k-\sigma_1-\sigma_j}^T & a^2_j I_{\sigma_j}
\end{bmatrix} 
$$
\end{enumerate}
\label{p:proposition:svd}
\end{proposition}

\begin{proof}
Standard diagonalization theory for the symmetric positive semidefinite matrices $YY^T$ and $Y^TY$ yield that $\textnormal{Card}(\{i : a^2_i > 0\} = \ell={\rm Rank}(A_Y)={\rm Rank}(Y^TY)=\min(n,k)$. Here, $1\leq i\leq j$. Also it shows that any orthogonal matrix $g=[\u_1 \u_2 \ldots \u_k]$ with eigenvectors for columns will diagonalize $A_Y$ with $i$'th diagonal element $$a^2_i \delta_{i,j}=<A_Y\u_i,\u_j>=<\u_i,A_Y\u_j>$$. Since we've indexed the eigenvalues in non-increasing order, we get the diagonalized form claimed. Finally if ${\rm det}(g)=-1$ then replacing any odd number $\u_i$ by $-\u_i$ yields $g_{modified}\in SO(k)\subseteq G$. 
\end{proof}

\section{Thin SVD: A natural map from images to ellipsoids.}
\setcounter{equation}{0}

In this last section, we need to introduce and study Thin SVD as a natural map from images to ellipsoids. This is given via Theorem~\ref{t:main4} below.

Once done, Theorem~\ref{t:main1}, Theorem~\ref{t:main2} and Theorem~\ref{t:main3}
are established.

Assume that $n\geq k$. Given an image $Y$ which is a $n\times k$ matrix, we may write $Y=U\Sigma V^T$ where $U$ is a $n\times n$ matrix with $U^TU=I_k$, $\Sigma$ is a diagonal $n\times k$ matrix consisting of singular values $\sigma_i\geq 0$ $(1\leq i\leq k)$ and $V$ is a $k\times k$ matrix.  Then the non-zero singular values $\sigma_i$ are the lengths of the sides of ellipsoids
$\left\{Y{\bf x}:\, ||{\bf x}||_2=1\right\}$ which are of dimension $l\leq k$.  If $r$ is such that $\sigma_1\geq ...\geq \sigma_r>\sigma_{r+1}=\ldots.=\sigma_k=0$, then ${\rm rank}(Y)=r$ and ${\rm range}(Y)={\rm span}({\bf  u}_1,...,{\bf u}_r)$. Here ${\bf u}_i \textrm{ is the $i$'th column of } U$. Thus the thin SVD produces a natural map from images to ellipsoids.   

We have:

\begin{theorem}
The thin SVD produces a map from all $\tilde{G}\ltimes \RR^k$ equivalence classes of $n$ pointed images of $\RR^k$ onto and 1-1 the set of all ellipsoids in  ${\1}_n^{\perp}\subseteq \RR^k$ of dimension $k$. \Ie, The images $Y$ for which ${\rm dim}{\rm [(Range)} (Y(I-1/nI_n^T I_n))]=l\leq k$ map onto ellipsoids of dimension $l\leq k$.  Finally since $g\in G$ acts on $Y$ via
\begin{eqnarray*}
&& g(Y)=[g{\bf y}_1^T, ..., g{\bf y}_n^T]^T\equiv \\
&& \equiv [\tilde{g}{\bf y}_1^T,..., \tilde{g}{\bf y}_n^T]^T+[1,...,1]^T{\bf w}, {\bf w}\in \mathbb R^k
\end{eqnarray*}    
we have
\[
g(Y)(I-1/n[1,...1]^T[1,...,1])=\tilde{g}(Y)(I-1/n[1,...1]^T[1,...,1]).
\]
ie, the $G$ equivalence classes of images will map onto the $\tilde{G}$ equivalence classes of ellipsoids.
In the case of the similarity group, we take $\tilde{G}=\mathbb R^{+}$
and say that two ellipsoids are equivalent provided one is a non constant multiple of the other. 
\label{t:main4}
\end{theorem}

\begin{proof} Notice that the map we use to map images to ellipsoids
\[
Y\to Y(I-1/n[1,...1]^T[1,...,1]=Y(I-1/nI_n^TI_n)=Y_{\rm norm}
\]
is a motion equivalent map. Now form the positive definite symmetric matrix $\tilde{A}_Y=Y_{\rm norm}Y_{\rm norm}^T \in \mathbb R^k$. This gives rise to a new semi inner product
\[
<{\x},{\z}>_{\tilde{A}_Y}:=<Y_{norm}^T{\bf x}, Y_{norm}^T{\bf z}>_{\mathbb R^k}.
\]
The ellipsoid associated to $Y$  is
\[
\tilde{E}_Y=\left\{{\bf x}\in \mathbb R^k:\, {\bf x}^T\tilde{A}_Y{\bf x}=1\right\}.
\]
The positive semi-definite matrix $\tilde{A}_Y$ has a factorization of the form
\[U{\rm diag}(\sigma_1^2,...,\sigma_l^2,0,...,0)U^T\] Here, $U^U-I$, ie $U$ is orthogonal $n\times n$, $l={\rm rank}(Y_{norm}Y_{norm}^T)$. The orthogonal vectors $\sigma_1{\bf u}_1,
\sigma_2{\bf u}_2,..., \sigma_l{\bf u}_l$ are the principle axes of the ellipsoid $E_Y$. If the $\sigma_i$ are distinct, then the principle axes are determined. If we order $\sigma_i$ so that
$\sigma_1\geq .....\sigma_l$ then any string of adjacent $=$(as opposed to $>$) means that only the span of the corresponding ${\bf u}_i$ is unique.        
Next, observe that $\tilde{E}_Y \subseteq I_n^{\perp}$ since if ${\bf x}\in E_y$, we have
\begin{eqnarray*}
&& <I_n,{\bf x}>=<I_n, A{\bf x}>=<Y_{\rm norm}^{T}I_n, Y_{\rm norm}^T {\bf x}> \\
&& =<{\bf 0}, Y_{\rm norm}^T{\bf x}>=0.
\end{eqnarray*}
Note that we have used that $Y_{\rm norm}^{T}I_n$=the sum of columns of $Y_{\rm norm}^T = 0$
the transpose of the sum of rows of $Y_{\rm norm}$=0.

Suppose now that $\tilde{E}$ is any ellipsoid $\subseteq I_n^{\perp}$ of dimension $k$. Then there exist mutually orthonormal vectors ${\bf u}_1,..., {\bf u}_l$ and lengths $\sigma_1...\sigma_l$ which give principle axes for $E$ as $\sigma_1{\bf u}_1,...,\sigma_l{\bf u}_l$. The $n$ pointed image
\[
Y=[\sigma_1u_1,....,\sigma_lu_l, 0,...,0]\]
has $\tilde{E}_Y=\tilde{E}$ ($Y$ has $k$ columns).
So the map is onto. Note that the map from all images to ellipsoids is not 1-1. We will need to take equivalence classes for this.     

Let us  deal with the 1-1.  We want to show that each motion equivalence class maps to one ellipsoid 1-1 and that if two images map to the same ellipsoid then the images are motion equivalent. 
\medskip

{\bf Step 1.} 
\medskip

If $g = (B, \u)\in 
\tilde{G} \ltimes \mathbb R^k, \textnormal{ so } \tilde{g} = B \in \tilde{G}$, then $E_{g(Y)}=E_{\tilde{g}(Y)}$. To see this, note that if
\[
g(Y)=YB^T+I_n\u^T
\]
then a straightforward calculation gives
\begin{eqnarray*}
(g(Y))_{norm} && = (I-1/nI_nI_n^T)(YB^T+I_n\u^T) \\
&& =Y_{norm}B^T + (I-1/nI_nI_n^T)I_n\u^T \\
&&= Y_{norm}B^T = \tilde{g}(Y)_{norm}.
\end{eqnarray*}
Thus,
\begin{eqnarray*}
&& A_{gY}=A_(gY)_{norm}((gY)_{norm})^T=Y_{norm}B^TBY_{norm}^T \\
&& =\tilde{g}(Y)_{norm}{\tilde{g}(Y)}_{norm}^T=A_{\tilde{g}(Y)}
\end{eqnarray*}
as $B\in O(K)$. Thus, $E_{gY}=E_{\tilde{g}(Y)}$. 
\medskip

{\bf Step 2.} 
\medskip

 Now let $Y^1$ and $Y^2$ be $n$ point images in $\mathbb R^k$ ($n\times k$) such that
$E_{Y^1}=E_{Y^2}$. Without loss of generality, we will assume that $Y^i(1/nI_n^TI_n)=0$ and the sum of the rows of $Y^i$, is zero, $i=1,2$. Now form the thin SVD of $Y^i$, $i=1,2$. This means that we have $n\times k$ orthogonal matrices $U_i$, $i=1, 2$, orthogonal $k\times k$ matrices $V_i$, $i=1,2$ and
$k\times k$ diagonal matrices $\Sigma_i={\rm diag}(\sigma_{1,i}, \sigma_{2,i},..., \sigma_{k, i}), \, i=1,2$
with all $\sigma_{j,i}\geq 0,\, j=1,...,k,\, i=1,2$ and with $Y^i=U_i\Sigma_iV_i^T$.  Let now for $i=1,2$, $l_i$ be the rank of
$\Sigma_i$ equals the number of non zero $\sigma_{j,i}$. The first $l_i$ columns of $U_i$ form an orthonoromal basis for the range of $Y^i$ and the ellipse $E_{ji}$ has dimension $l_i$ with the non zero columns of $U_i\Sigma_i$ as principle axes.
So $\tilde{E}_{Y_1}=\tilde{E}_{Y_2}$ implies that $l_1=l_2$ and $\sigma_{j,1}=\sigma_{j,2}$, $j=1,......,k$. So $\Sigma_1=\Sigma_2$, call it $\tilde{\Sigma}$.   Now let
\[
O_{\tilde{\Sigma}}=\left\{O\in O_k:\, O\tilde{\Sigma}=\tilde{\Sigma}O\right\}
\]
where $O_k$ denotes the space of all $k\times k$ real square matrices. We now define numbers $\tau_i$ $1\leq i\leq k$ inductively as follows.
Set $\tau_1=\sigma_1$. Let $m_1={\rm card}\left\{\sigma_j:\, \sigma_j=\tau_1\right\}$. Now define $\tau_2=\sigma_{1+m_1}{\rm card}{\tau_1}$ and $m_2={\rm card}\left\{\sigma_j:\, \sigma_j=\tau_2\right\}$. Now define $\tau_3$ and $\tau_i,\,3< i<k$ inductively.

Then
\[
O_{\tilde{\Sigma}}={\rm diag}[O_{m_1}, O_{m_2},......].
\]
Note that $O_{\tilde{\Sigma}}$ which is $k\times k$ is defined blockwise. We can also write
\[
\tilde{\Sigma}={\rm diag}[\tau_{1}I_{m_1},...., \tau_{k}I_{m_k}]
\]
Note that  $\tilde{\Sigma}$  which is $k\times k$ is defined blockwise.   Now the principal axes of $\tilde{E}_{Y_1}=\tilde{E}_{Y_2}$ of length $\tau_i$ span a subspace $w_i$ of dimension $m_i$. In particular the $m_i$ columns of $U_j (j=1,2)$ corresponding to the block
$\tau_iI_{m_i}$ are an orthonormal basis for $w_i$ (one for $U_1$ and one for $U_2$). Let us define a positive integer $p$

as follows: Let $\sigma_{j^*}$ be the last non zero singular value and let $p$ be that positive integer where $\tau_p=\sigma_{j^*}$ and $\tau_{p+1}=0$. Then $p$ is unique and corresponds to positive principle axes. Now break
$U_j$ into $k$ blocks of columns so that for $j=1,2$
\[
U_j=[{U}_{j,m_1},....,{U}_{j,m_p}, {U}_{j, p+1},...,{U}_{j,k}]
\]
and observe that  we then have $R(U_{1,i})=R(U_{1,i})=W_i,\, 1\leq i\leq k$ or the columns of $V_i$ satisfy a similar relationship.
Using the above, it is not difficult to deduce that $Y^1$ and $Y^2$ are orthogonally equivalent.
\end{proof}

\section{Concluding remarks.}
\setcounter{equation}{0}

This paper does not  develop algorthims for specific  metric computation on different spaces of  real-life images or signals. However it is clear that this is a natural next step to take
for numerous applications for example in manifold and topological learning.

\section{Appendix.}

\subsection{The multiplication action of  
$GL(k)\ltimes \RR^k$ on $\RR^k$.}

Let us examine the multiplication action of  
$GL(k)\ltimes \RR^k$ on $\RR^k$. Indeed, if $B\in GL(k)$, $\x$ and $\w \in \RR^k$, then
\[
(B,{\w})\x=B{\x}+{\w}.
\]
So,
\begin{eqnarray*}
 (A,\u) \circ (B,\w)(\x)&&= (A,\u)(B\x+\w) \\
 =AB(\x)+(A\w+\u)&&=(AB,A\w+\u)(\x).
\end{eqnarray*}
Since $GL(k)$ is a group, $(AB)^{-1}=B^{-1}A^{-1}$, $AB\in GL(k)$ and hence $(A,{\u})o(B,{\w}) = (AB, \u + A\w)\in GL(k)\ltimes \RR^k$.
In the case of $O(k)\ltimes \RR^k$,  if $A,B\in O(k)$, then $(AB)^T(AB)=B^TA^TAB=Id$ so again
$(A,{\bf u})\circ (B,{\bf w})\in  O(k)\ltimes \RR^k$.

\subsection{The affine group on $\RR^k$: A matrix definition.}

An isomorphic "matrix" definition proves useful for the affine group on $\RR^k$.  Let $ \RR^k $ be isometrically embedded as a hyperplane in $ \RR^{k+1}$ via $$\RR^k \owns \v \hookrightarrow \begin{bmatrix}
\v \\
0
\end{bmatrix} \in \RR^{k} \oplus \RR\e_{k+1}$$  and the group $Aff(k)$ be realized 
as the subgroup $\textnormal{ of }GL(k+1)$ given by the image of the bijection 
$$
Aff(k) \xrightarrow{\simeq} \begin{bmatrix}
GL(k) & \RR^k \\
\0^T  &  1
\end{bmatrix} : (A,\u) \mapsto \begin{bmatrix}
 A   & \u \\
\0^T &  1
\end{bmatrix}. 
$$
Then the action of $Aff(k)$ on  $\RR^k$ is given by matrix multiplication under these embeddings since
$$
(A, \u)(\v) \mapsto 
\begin{bmatrix}
  A  & \u \\
\0^T &  1
\end{bmatrix}
\begin{bmatrix}
\v \\
1
\end{bmatrix}
=
\begin{bmatrix}
A\v + \u \\
1
\end{bmatrix} \reflectbox{$\mapsto$} {A\v+\u}
$$

\subsection{Action of $Aff(k)$ on $M(n,k)$.}

Since $Y\in M(n,k)$ has rows $\{{\y}_1^T, \ldots, {\y}_n^T\}$ with ${\y}_i \in {\RR}^k$, the action of $Aff(k)$ on $M(n,k)$ is realized via the embedding $Y \mapsto [Y 1_n] \in M(n, k+1)$ and the transpose of the matrix product shown above applied rowwise :
$$\begin{bmatrix}
Y & 1_n
\end{bmatrix}  \begin{bmatrix}
  A  & \u \\
\0^T &  1
\end{bmatrix}^T =
\begin{bmatrix}
Y & 1_n
\end{bmatrix}  \begin{bmatrix}
  A^T  & \0 \\
\u^T &  1
\end{bmatrix} = [YA^T + 1_n\u^T]$$.
Note that the action  on the right of $M(n,k)$ is via the inverse action, i.e. $Y\mapsto Y(A,\u){^-1} =YA^{-1}-\1_nA^{-1}\u^T$.

\subsection{Singular value decomposition (SVD).}

The singular value decomposition of an $m\times n$ complex matrix $M$ is a factorization of the form $U\Sigma V^{*}$, where $U$ is an $m\times m$ complex unitary matrix, $\Sigma$ is an $m\times n$ rectangular diagonal matrix with non-negative real numbers on the diagonal, and $V$ is an $n\times n$ complex unitary matrix. If $M$ is real, then $U$ and ${V} ^{T}=V^{*}$ are real orthogonal matrices.

The diagonal entries $\sigma _{i}$=$\Sigma_{ii}$ of $\Sigma$ are the singular values of $M$. The number of non-zero singular values is ${\rm rank}( M)$. The columns of $U$ and the columns of $V$ are  left-singular vectors and right-singular vectors of $M$, respectively.
The SVD is not unique.


\begin{thebibliography}{99}
\bibitem{D7} S. B. Damelin, {\it A walk-through energy, discrepancy, numerical integration and group invariant measures on measurable subsets of Euclidean space}, Numerical Algorithms, {\bf 48}, (1-3) (2008), pp. 213-235. 
\bibitem{D6} S. B. Damelin, {\it On the Whitney extension problem for near isometries and beyond}, arXiv: 2103.09748, submitted for consideration for publication.
\bibitem{DR11} S. B. Damelin; F. Hickernell; D. Ragozin; X. Zeng, {\it On energy, discrepancy and g-invariant measures on measurable subsets of Euclidean space,} Journal of Fourier Analysis and its Applications {\bf 16} (2010), pp. 813-839.
\bibitem{W} S. B. Damelin and W. Miller, {\it Mathematics and Signal Processing}, Cambridge Texts in Applied Mathematics (No. 48) February 2012..
\bibitem{H65} K. Hamm, {\it Nonuniform Sampling and Recovery of Bandlimited Functions in Higher Dimensions}, Journal of Mathematical Analysis and Applications, {\bf 450} Issue 2 (2017), pp. 1459-1478.
\bibitem{Ki} J. Kileel, {\it Algebraic geometry for computer vision}, PhD thesis, University of Berkeley, 2017.
\bibitem{Ki34} J. Kileel, {\it Minimal problems for the calibrated trifocal variety}, SIAM Journal on Applied Algebra and Geometry {\bf 1} (2017), pp. 575-598.
\bibitem{Ki51} J. Kileel, {\it Subspace power method for symmetric tensor decomposition and generalized PCA}, arXiv:1912.04007.
\bibitem{Ki1} J. Kileel; Z. Kukelova; T. Pajdla; B. Sturmfels, {\it Distortion varieties}, Foundations of Computational Mathematics {\bf 18} (2018), pp. 1043–1071. 
\bibitem{Lii1} R. Lederman, R; J. Andén; A. Singer, {\it Hyper-Molecules: on the Representation and Recovery of Dynamical Structures, with Application to Flexible Macro-Molecular Structures in Cryo-EM}, Inverse Problems, 2019.
\bibitem{O} P. J. Olver, {\it Invariant signatures for recognition and symmetry}, I.M.A., University of Minnesota, April, 2006.
\bibitem{O1} P. J. Olver; G. Sapiro; A. Tannenbaum, {\it Affine invariant detection: edge maps, anisotropic diffusion, and active contours}, Acta Appl. Math, {\bf 59} (3) (1999), pp. 45–77.
\bibitem{O2} P. J. Olver; G. Sapiro; A. Tannenbaum, {\it Invariant geometric evolutions of surfaces and volumetric smoothing}, SIAM J. Appl. Math, {\bf 57} (1) (1997), pp. 176–194.
\bibitem{W1} M. Werman, {\it Affine invariants}, Computer Vision: A Reference Guide, 2014.
\bibitem{W3} M. Werman; E. Begelfor, {\it Affine invariance revisited}, CVPR, 2006.
\bibitem{W8} M. Werman; D. Weinshall, {\it Similarity and affine invariant distance between point sets}, PAMI {\bf 17} (8), pp. 810-814.
\end{thebibliography}
\end{document}